\pdfoutput=1

\documentclass[11pt]{article}
\usepackage[utf8]{inputenc}

\usepackage{fullpage}
\usepackage{mathtools}
\usepackage{amsthm}
\usepackage{amsfonts}
\usepackage{enumitem}
\usepackage{hyperref}
\usepackage{xcolor}
\usepackage{graphicx}
\usepackage{float}
\usepackage[algo2e,ruled,linesnumbered]{algorithm2e}
\usepackage{authblk}
\hypersetup{
	colorlinks	= true,
	urlcolor	= orange,
	citecolor	= blue,
	linkcolor	= red
}

\title{Stronger Calibration Lower Bounds via Sidestepping\thanks{We would like to thank Dean P.\ Foster for bringing to our attention this calibration perspective on online predictions as well as the problem of proving super-$\sqrt{T}$ calibration lower bounds, and for his comments and suggestions on an earlier draft of this paper. This work was supported by NSF Awards CCF-1704417 and AF-1813049, DOE Award DE-SC0019205, and ONR Young Investigator Award N00014-18-1-2295.}}
\date{}
\author{Mingda Qiao}
\author{Gregory Valiant}
\affil{\texttt{\{mqiao,valiant\}@stanford.edu}}
\affil{Stanford University}

\newcommand{\A}{\mathcal{A}}    
\newcommand{\Ber}{\mathrm{Ber}}
\newcommand{\calerr}{\mathrm{calerr}}   
\newcommand{\Ecover}{\event^{\textrm{cover}}}   
\newcommand{\Enegl}{\event^{\textrm{negl}}}     
\newcommand{\Epoch}{\mathsf{Epoch}}
\newcommand{\Etruth}{\event^{\textrm{truth}}}   
\newcommand{\event}{\mathcal{E}}    
\newcommand{\Ex}[2]{\operatorname*{\mathbb{E}}_{#1}\left[#2\right]}  
\newcommand{\game}{\textsf{Sign-Preservation}}      
\newcommand{\I}[1]{\mathbb{I}\left[#1\right]}       
\newcommand{\Int}{\mathcal{I}}      
\newcommand{\maxerr}{\mathrm{maxerr}}       
\newcommand{\opt}{\mathrm{opt}}     
\newcommand{\poly}{\operatorname*{poly}}    
\newcommand{\pr}[1]{\Pr\left[#1\right]}     

\newcommand{\SPinner}{\mathsf{SP}^{\textrm{inner}}}
\newcommand{\SPouter}{\mathsf{SP}^{\textrm{outer}}}
\newcommand{\Tact}{T^{\mathrm{actual}}}     

\newtheorem{theorem}{Theorem}
\newtheorem{definition}[theorem]{Definition}
\newtheorem{lemma}[theorem]{Lemma}
\newtheorem{proposition}[theorem]{Proposition}
\newtheorem{example}{Example}

\begin{document}

\maketitle

\begin{abstract}
    We consider an online binary prediction setting where a forecaster observes a sequence of $T$ bits one by one. Before each bit is revealed, the forecaster predicts the probability that the bit is $1$. The forecaster is called well-calibrated if for each $p \in [0, 1]$, among the $n_p$ bits for which the forecaster predicts probability $p$, the actual number of ones, $m_p$, is indeed equal to $p \cdot n_p$. The \emph{calibration error}, defined as $\sum_p |m_p - p n_p|$, quantifies the extent to which the forecaster deviates from being well-calibrated. It has long been known that an $O(T^{2/3})$ calibration error is achievable even when the bits are chosen adversarially, and possibly based on the previous predictions. However, little is known on the lower bound side, except an $\Omega(\sqrt{T})$ bound that follows from the trivial example of independent fair coin flips.
    
    In this paper, we prove an $\Omega(T^{0.528})$ bound on the calibration error, which is the first super-$\sqrt{T}$ lower bound for this setting to the best of our knowledge. The technical contributions of our work include two lower bound techniques, \emph{early stopping} and \emph{sidestepping}, which circumvent the obstacles that have previously hindered strong calibration lower bounds. We also propose an abstraction of the prediction setting, termed the \game{} game, which may be of independent interest.  This game has a much smaller state space than the full prediction setting and allows simpler analyses. The $\Omega(T^{0.528})$ lower bound follows from a general reduction theorem that translates lower bounds on the game value of \game{} into lower bounds on the calibration error.
\end{abstract}

\newpage

\section{Introduction}
We study the following online binary prediction problem. A forecaster predicts a binary sequence of length $T$ that is observed one bit at a time. Before seeing each bit, the forecaster makes a prediction about the probability that this bit is a ``1''. For simplicity, we require all the predictions to fall in a finite set $P \subset [0, 1]$ specified by the forecaster at the beginning. At the end of the $T$ time steps, the \emph{calibration error}\footnote{More generally, the $\ell_q$ calibration error is defined as $\left(\sum_{p\in P}\frac{n_p(T)}{T}\left|\frac{m_p(T)}{n_p(T)} - p\right|^q\right)^{1/q}$ in the literature. Up to a factor of $T$, the definition in~\eqref{eq:calerr} coincides with the $\ell_1$ calibration error, which is also called the expected calibration error (ECE).} incurred by the forecaster is defined as
\begin{equation}\label{eq:calerr}
    \calerr(T) \coloneqq
    \sum_{p \in P}|m_p(T) - n_p(T)\cdot p|,
\end{equation}
where $n_p(T)$ denotes the number of times probability $p$ is predicted up to time $T$, and $m_p(T)$ is the number of ones observed among those $n_p(T)$ time steps. Thus, the calibration error quantifies the extent to which the forecaster's predictions are well-calibrated in the sense that for every possible prediction value $p$, the frequency of $1$ among the time steps at which $p$ is predicted is indeed close to $p$.

The notion of calibration is incomparable to other usual performance metrics such as prediction accuracy.  Particularly when predicting potentially noisy binary outcomes, it is difficult to establish good benchmarks for prediction accuracy, as it is generally impossible to distinguish between noise in the observations, versus a failure of the predictor.  By contrast, calibration is a natural desiderata that applies whether or not the observations have intrinsic noise.
Being well-calibrated can be viewed as a minimum requirement on the forecaster for its predictions to be interpreted as meaningful probabilities: Suppose that among all the days on which a weather forecast predicted a $50\%$ chance of rain, it rained on only $10\%$ of the days. The predictions of such weather forecasts clearly lack credibility.  

There has been a recent surge of interest in calibration, both from the machine learning community (e.g.,~\cite{kuleshov2015calibrated,guo2017calibration,kumar2019verified,zhao2020individual}), and from the perspective of algorithmic fairness (e.g.,~\cite{kleinberg2017inherent,pleiss2017fairness,hebert2018calibration,romano2019malice,shabat2020sample,jung2020moment}).  In the machine learning setting, this study is motivated in part by the fact that trained neural networks are often spectacularly poorly calibrated and overconfident in their predictions.  The connection between calibration and fairness is especially natural: as proposed in \cite{kleinberg2017inherent}, a fair classifier should be calibrated on every protected demographic group. Indeed, predictions that are not well-calibrated for some demographic groups would seem to conflict with the most intuitive notion of what it means to treat all groups fairly.

The calibration aspect of online predictions was first studied by Foster and Vohra~\cite{foster1998asymptotic}. They gave a randomized forecaster that achieves an $O(T^{2/3})$ calibration error in expectation, even if the $T$ bits are chosen by an adaptive adversary that chooses the $t$-th bit based on the bits and predictions in the previous $t-1$ steps. This $O(T^{2/3})$ upper bound has a simple non-constructive proof due to Sergiu Hart based on the minimax theorem~\cite[Section 4]{foster1998asymptotic}: For each fixed finite set $P$, each deterministic adaptive strategy of the forecaster (resp.\ the adversary) can be viewed as a function that maps $\bigcup_{t=0}^{T-1}(\{0,1\}^t\times P^t)$ to $P$ (resp.\ $\{0, 1\}$), so there are only finitely many such strategies. Thus, by the minimax theorem, it suffices to prove the following claim: Against any given adversary (which might be randomized and adaptive), there is a forecaster that achieves an $O(T^{2/3})$ calibration error in expectation. This claim, in turn, can be proved using the following ``rounding strategy'': (1) choose $P = \left\{0, \frac{1}{T^{1/3}}, \frac{2}{T^{1/3}}, \ldots, 1\right\}$; (2) at each time step $t$, compute the probability of the event $b(t) = 1$ conditioning on the previous $t-1$ steps (using the knowledge of the given adversary), and predict the value in $P$ that is closest to this conditional probability. We refer the readers to~\cite{hart2020calibrated} for further details of this proof.

On the other hand, less is understood on the lower bound side. The only known lower bound on the calibration error is $\Omega(\sqrt{T})$, which can be proved using a simple adversary that outputs $T$ independent and uniformly distributed random bits. In this case, the optimal strategy is to predict probability $1/2$ at every step $t$. Then, the calibration error $\calerr(T)$ reduces to $|m_{1/2}(T) - T/2|$, where $m_{1/2}(T)$ follows the binomial distribution $B(T, 1/2)$, and this error is $\Omega(\sqrt{T})$ in expectation. Unfortunately, there is no known scheme of the adversary that outperforms the trivial one (that outputs independent coin flips) and gives a $\omega(\sqrt{T})$ bound on the  calibration error, not to mention an $\Omega(T^{2/3})$ bound that matches the best known upper bound.

\subsection{The Prediction Setting}
The binary prediction setting is formally defined as a two-player multi-stage game between a forecaster and an adversary. The forecaster first specifies a finite set $P \subset [0, 1]$ from which the predictions are selected. At each time step $t = 1, 2, \ldots, T$, the forecaster chooses $p(t) \in P$ and the adversary chooses $b(t) \in \{0, 1\}$ simultaneously. Both choices may depend on the previous $t-1$ steps but not the other player's action at time $t$. For each $p \in P$, let $n_p(t) \coloneqq \sum_{i=1}^{t}\I{p(i) = p}$ denote the number of times that $p$ is predicted by the forecaster during the first $t$ time steps, and let $m_p(t) \coloneqq \sum_{i=1}^{t}\I{p(i) = p \wedge b(i) = 1}$ denote the number of time steps at which $p$ is predicted and the bit chosen by the adversary is $1$. Then, the cumulative calibration error up to time $t$ is defined as
\[
    \calerr(t) \coloneqq
    \sum_{p \in P}|m_p(t) - n_p(t)\cdot p|.
\]
Define $\Delta_p(t) \coloneqq m_p(t) - n_p(t)\cdot p$ as the total bias associated with prediction value $p$ after the first $t$ time steps. Moreover, let $\Delta^{+}_p(t) \coloneqq \max(\Delta_p(t), 0)$ and $\Delta^{-}_p(t) \coloneqq \max(-\Delta_p(t), 0)$ denote the positive and negative parts of $\Delta_p(t)$. Then, $\calerr(t)$ can be equivalently written as
\[
    \calerr(t)
=   \sum_{p \in P}|\Delta_p(t)|
=   \sum_{p \in P}\Delta^{+}_p(t) + \sum_{p \in P}\Delta^{-}_p(t).
\]
For each quantity that is labeled by a time step (e.g., $\calerr(t)$ and $\Delta_p(t)$), we may omit the parameter $t$ (and simply write, e.g., $\calerr$ and $\Delta_p$) if it can be inferred from the context. In particular, we will drop the notation $t$ when describing a scheme of the adversary, since the time $t$ is naturally given by the time step at which that statement is executed.

It should be noted that the finiteness assumption on $P$ is not too restrictive and is standard in the literature (e.g.,~\cite{foster1998asymptotic}). This assumption can be justified by real-world prediction scenarios such as weather forecasts, where the probability of precipitation is typically rounded to $5\%$ or $10\%$ increments. Moreover, we can verify that rounding each prediction $p(t)$ to the nearest multiple of $1/T$ would increase $\calerr(T)$ by at most an additive constant. Thus, it is without loss of generality to take $P = \{0, 1/T, 2/T, \ldots, 1\}$.

\subsection{Obstacles Against Strong Lower Bounds}\label{sec:obstacle}
Recall that an $O(T^{2/3})$ upper bound can be proved by analyzing a forecaster that predicts the nearest multiple of $1/T^{1/3}$ to the conditional probability that the next bit is $1$. Suppose that the adversary divides the time horizon $T$ into $k = T^{1/3}$ ``epochs'' of length $T/k$, and outputs $T/k$ independent samples from the Bernoulli distribution $\Ber(i/k)$ in the $i$-th epoch. Then, the forecaster with the rounding strategy would keep predicting probability $i/k$ in the $i$-th epoch, and the expected calibration error is given by
\[
    \Ex{}{\calerr(T)}
=   \sum_{i=1}^{k}\Ex{X\sim B(T/k, i/k)}{|X - (T/k)\cdot(i/k)|}
=   \Omega(k\cdot \sqrt{T/k}) = \Omega(T^{2/3}),
\]
where $B(\cdot, \cdot)$ denotes the binomial distribution. This indicates that the analysis of the $O(T^{2/3})$ upper bound is tight.  Assuming that the forecaster is ``truthful'' in the sense that its prediction is always equal to (or very close to) the conditional probability of the next bit, the above example also suggests an $\Omega(T^{2/3})$ lower bound for all such truthful forecasters. 

Unfortunately, this truthfulness assumption on the forecaster does not always hold; in various scenarios the forecaster, to minimize the calibration error, has an incentive to make untruthful predictions that are far away from the true probabilities. In the following, we describe several such scenarios, including \emph{coarse-grained binning} and \emph{cover-up}, that make the construction of lower bound schemes highly nontrivial.

The first example shows that, while the above construction proves the tightness of the upper bound analysis, there exists another simple forecaster that achieves a small error on it.
\begin{example}[Coarse-grained binning]\label{ex:coarse}
    Let us revisit the case that the binary sequence consists of $T/k$ independent samples from each of $\Ber(1/k), \Ber(2/k), \ldots, \Ber(k/k)$ for $k = T^{1/3}$. Note that the sum of the $T$ bits has an expectation of $T \cdot \frac{k+1}{2k}$ and an $O(T)$ variance. Therefore, if the forecaster predicts $\frac{k+1}{2k}$ at each of the $T$ steps, the resulting calibration error is $O(\sqrt{T})$ in expectation.
\end{example}

In Example~\ref{ex:coarse}, while we expect the forecaster to put the $T$ bits into $k$ ``bins'' associated with probabilities $1/k, 2/k, \ldots, k/k$ faithfully and incur an $\Omega(T^{2/3})$ error, the forecaster would actually merge all these bins into a larger, coarse-grained bin corresponding to probability $\frac{k+1}{2k}$. More generally, as long as the $T$ bits are independently drawn with fixed probabilities $p^*_1, p^*_2, \ldots, p^*_T$, the forecaster may as well predict the average $\frac{1}{T}\sum_{t=1}^{T}p^*_t$ at every single time step, resulting in $\Ex{}{\calerr(T)} = O(\sqrt{T})$.

\begin{example}[Cover-up]\label{ex:cover}
    Suppose that the sequence consists of $T/3$ uniformly random bits followed by $T/3$ ones and then $T/3$ zeros. Moreover, suppose that the forecaster predicts $1/2$ in each of the first $T/3$ steps. Then, the calibration error after the first $T/3$ steps is $\Omega(\sqrt{T})$ in expectation. However, the forecaster can always ``cover up'' this error using the subsequent bits: If the first $T/3$ bits contain more zeros than ones, the forecaster may keep predicting $1/2$ (even though the bits are known to be ones) until $m_{1/2}(t) = n_{1/2}(t)/2$ at some point $t$. Similarly, the forecaster may cover up the error by predicting $1/2$ during the last $T/3$ time steps, if ones outnumber zeros among the first $T/3$ bits.
\end{example}

In Example~\ref{ex:cover}, the forecaster can always achieve a zero calibration error by untruthfully predicting $1/2$ for bits that are known to be deterministic. While the example might appear a bit contrived, this phenomenon that a forecaster can strategically decrease the cumulative calibration error by predicting untruthfully is not uncommon. Foster and Hart~\cite{foster2020forecast} refer to such behavior as ``backcasting'' (in contrast to forecasting), in the sense that the forecaster makes use of the future outcomes to disguise the mistakes it has made in the past.

The following example, termed ``forecast hedging'' in~\cite{foster2020forecast}, indicates that the cover-up scenario is universal and makes it difficult to prove strong calibration lower bounds.

\begin{example}[Forecast hedging]\label{ex:hedging}
    Suppose that at time $t$, it holds that $\Delta_{p_1}(t) \le -1$ and $\Delta_{p_2}(t) \ge 1$ for some $p_1 < p_2$. We claim that the forecaster can decrease the calibration error in expectation after the next time step (i.e., ensure that $\Ex{}{\calerr(t+1)} \le \calerr(t)$) by predicting either $p_1$ or $p_2$, each with probability $1/2$.
    
    To see this, first suppose that the next bit is $0$. Then, with probability $1/2$, $\Delta_{p_1}$ is decreased by $p_1$, and $\calerr$, which contains a $|\Delta_{p_1}|$ term, will be increased by $p_1$; with the remaining probability $1/2$, $\Delta_{p_2}$ is decreased by $p_2$, and $\calerr$ also drops by $p_2$. In expectation, the cumulative calibration error drops by $\frac{p_2 - p_1}{2} > 0$. A similar analysis works for the case that the next bit is $1$. Thereby, the forecaster can cancel out part of the previous error by randomizing between the two predictions $p_1$ and $p_2$, without taking into account the actual probability of the next bit. 
\end{example}

\subsection{Our Results}\label{sec:result}
The main result of this paper is the first super-$\sqrt{T}$ lower bound on the calibration error in the online binary prediction setting.

\begin{theorem}\label{thm:main}
    Let $\alpha = \frac{\log 8}{\log 255}$, $\beta = \frac{\log(9/2)}{\log 255}$, and $c = \frac{2\beta + 1}{\alpha + 2\beta + 2} > 0.528$.
    There exists a scheme of the adversary such that every forecaster incurs an $\Omega(T^c/\sqrt{\log T}) = \Omega(T^{0.528})$ calibration error in expectation.
\end{theorem}

The proof of Theorem~\ref{thm:main} builds on two simple yet powerful lower bound techniques tailored to calibration error, termed as \emph{early stopping} and \emph{sidestepping}, that manage to overcome the obstacles discussed in Section~\ref{sec:obstacle}.

\paragraph{Early stopping.} To prevent the forecaster from putting all bits into a single coarse-grained bin as in Example~\ref{ex:coarse}, we observe that to do this, the forecaster would likely encounter a large $\calerr(t)$ in the middle of the time horizon. For instance, suppose that the forecaster predicts $\frac{k+1}{2k}$ at every time step in Example~\ref{ex:coarse}. Then, the calibration error $\calerr(t)$ reduces to $|\Delta_p(t)| = |m_p(t) - n_p(t)\cdot p|$ for $p = \frac{k+1}{2k}$. Since each of the first $T/4$ bits has an expectation of at most $1/4$, the expected sum of these bits, $\Ex{}{m_p(T/4)}$, is at most $T/16$. On the other hand, $n_p(T/4)\cdot p = pT/4 \ge T/8$. It follows that $\calerr(T/4)$ will be as large as $\Omega(T)$ in expectation. Then, if the adversary deviates from the above construction and keeps outputting zeros in the remaining $3T/4$ time steps, $\calerr(T)$ will also be large.

This observation motivates the following ``early stopping'' trick: instead of directly lower bounding $\calerr(T)$, it suffices to show that $\calerr(t)$ is large at some step $t \in [T]$. Formally, define $\maxerr(t) \coloneqq \max_{t'\in[t]}\calerr(t')$ as the maximum cumulative error encountered during the first $t$ steps. The following proposition states that a high-probability lower bound on $\maxerr(T)$ implies the existence of another scheme that gives a high-probability lower bound on $\calerr(T)$.
\begin{proposition}\label{prop:early-stop}
    Suppose that for $B, p > 0$, there exists a scheme $\A$ of the adversary that spans at most $T$ time steps such that $\pr{\maxerr(\Tact) \ge B} \ge p$ holds for any forecaster, where random variable $\Tact$ denotes the number of steps that $\A$ actually lasts. Then, there also exists a scheme that lasts exactly $T$ time steps such that $\pr{\calerr(T) \ge B/2} \ge p$ for any forecaster.
\end{proposition}
\begin{proof}
    We define another scheme $\A'$ that simulates $\A$. As soon as $\calerr(t_0) \ge B$ holds at some time $t_0$, $\A'$ deviates from $\A$ and computes $\sum_{p\in P}\Delta^+_p(t_0)$ and $\sum_{p\in P}\Delta^-_p(t_0)$. Since the two terms sum up to $\calerr(t_0)$, at least one of the terms is at least $B/2$. If $\sum_{p\in P}\Delta^+_p(t_0) \ge B/2$, scheme $\A'$ keeps outputting $1$ in the remaining $T - t_0$ time steps; otherwise $\A'$ keeps outputting $0$. 
    
    In the former case, $\sum_{p\in P}\Delta^+_p$ will never drop below $\sum_{p\in P}\Delta^+_p(t_0)$, so we have
        \[\calerr(T) \ge \sum_{p\in P}\Delta^+_p(T) \ge \sum_{p\in P}\Delta^+_p(t_0) \ge B/2;\] 
    similarly, $\calerr(T) \ge \sum_{p\in P}\Delta^-_p(T) \ge \sum_{p\in P}\Delta^-_p(t_0) \ge B/2$ holds in the latter case. This shows that $\maxerr(\Tact) \ge B$ when running scheme $\A$ implies that $\calerr(T) \ge B/2$ when running scheme $\A'$, and thus proves the proposition.
\end{proof}

\paragraph{Sidestepping.} To prevent the forecaster from covering up the mistakes in the past, we note that such cover-ups are only possible if the probabilities of the later bits are in the ``right direction'' compared to the signs of $\Delta_p$'s. More concretely, in Example~\ref{ex:cover}, it is crucial that the last $2T/3$ bits contain both zeros and ones for the cover-up to be possible. In contrast, if $\Delta_{1/2}(T/3) > 0$ and the remaining bits are all ones, predicting $1/2$ will only further increase $\Delta_{1/2}$ and thus increase the calibration error.

This motivates us to choose the probabilities in a \emph{sidestepping} way, so that the error incurred by previous predictions cannot be fixed in the future. Suppose the adversary first flips a few fair coins with probability $1/2$. Then, assuming that all of the forecaster's predictions are exactly $1/2$, the adversary switches to another probability based on the sign of $\Delta_{1/2}(t)$. If $\Delta_{1/2}(t) > 0$, the adversary switches to a coin with a larger bias $3/4$, so that if the forecaster keeps predicting $1/2$, $\Delta_{1/2}$ will only further increase in expectation; otherwise, if $\Delta_{1/2} < 0$, the bias is changed to $1/4$ accordingly. Similarly, after tossing the coin with probability $3/4$ for a while, the adversary changes the probability to either $5/8$ or $7/8$ depending on the sign of $\Delta_{3/4}$.

We could repeat the above procedure and choose the probabilities such that cover-ups are not possible. However, as soon as we change the probability of the bit $\Theta(\log T)$ times, all the valid choices of the probability would fall into an interval of length $1/T$, at which point the forecaster can afford to keep predicting the same probability later on, since rounding the predictions to the nearest multiple of $1/T$ only increases the calibration error by an additive $O(1)$ amount. Thus, applying the above scheme verbatim could only force the forecaster into predicting at most $k = O(\log T)$ different values, each corresponding to an epoch with $T/k$ steps. Consequently, the resulting lower bound will be at best $\Omega(\sqrt{T\log T})$, which is not significantly better than the trivial bound.

Nevertheless, the actual construction of the scheme uses a similar strategy based on the idea of sidestepping. The key difference is that, instead of ensuring that the error incurred in \emph{every} epoch cannot be covered up in later epochs, the actual construction only guarantees this for a substantial fraction of the epochs, which also turns out to be sufficient for proving the lower bound.

\subsection{Related Work}
The notion of calibration in the prediction context dates back to at least the 1950s. In the literature of meteorology, Brier~\cite{brier1950verification} suggested that the quality of weather forecasts should be evaluated by comparing the forecast probability of rain and the actual fraction of rainy days among the days on which the probability is predicted. Calibration was later studied by  Dawid~\cite{dawid1982well} from a Bayesian perspective. 

Foster and Vohra~\cite{foster1998asymptotic} studied the online prediction of arbitrary binary sequences from the calibration perspective, and proved the existence of a forecaster that is well-calibrated on any arbitrary binary sequence. While the results in the paper were stated in the asymptotic regime where $T$ tends to infinity, the minimax proof due to Sergiu Hart (\cite[Section 4]{foster1998asymptotic}~and~\cite{hart2020calibrated}) directly implies an $O(T^{2/3})$ upper bound on the calibration error defined in~\eqref{eq:calerr}. The work of Foster and Vohra was later simplified by~\cite{fudenberg1999easier,foster1999proof} and extended to settings where the calibration condition is tested on different subsets of the time horizon~\cite{lehrer2001any,sandroni2003calibration,foster2011complexity}. Vovk~\cite{vovk2007non} further developed this approach and obtained non-asymptotic results.

The notion of calibration has also received increasing attention in the machine learning literature; see, e.g., \cite{kuleshov2015calibrated,guo2017calibration,kumar2019verified,zhao2020individual}. In binary classification, a classifier that maps data points to values in $[0, 1]$ is called well-calibrated if, among the data points on which value $p$ is predicted, the fraction of positive examples is close to $p$. In other words, the outputs of well-calibrated classifiers can be interpreted as the probability that the data points belong to the positive class. One reason for the increased attention on calibration is that trained neural networks typically yield very poorly calibrated models.  This classification setting is different from the online setting studied in this work, since the classifier makes the predictions for all the data points in a single batch. Thus, unlike the discussion in Example~\ref{ex:cover}, it is impossible to cover up the error incurred on certain data points by strategically adjusting the remaining predictions.

Calibration has also been recently studied in the setting of algorithmic fairness ~\cite{kleinberg2017inherent,pleiss2017fairness,hebert2018calibration,romano2019malice,shabat2020sample,jung2020moment}. In this context, a predictor labels each individual from the population with a value in $[0, 1]$, which is intended to be the probability that the individual belongs to a specific class of interest. The calibration criterion proposed by~\cite{kleinberg2017inherent} requires the predictions to be calibrated on a specified family of subsets of the population. When each subset in the family denotes a protected subset of the population, the calibration constraint prevents predictors that are discriminatory across different groups. \cite{hebert2018calibration} introduced another related fairness notion, \emph{multicalibration}, which requires the predictions to be well-calibrated on every subgroup of the population that can be identified computationally.

\subsection{Organization of the Paper}
In the remainder of the paper, we first take a detour and introduce a two-player game called \game{} in Section~\ref{sec:sign}. The \game{} game serves as a simplified abstraction of the sidestepping technique described in Section~\ref{sec:result}. We will state a reduction theorem (Theorem~\ref{thm:reduction}) in Section~\ref{sec:sign} and apply it to derive Theorem~\ref{thm:main}.

The rest of the paper is devoted to the proof of Theorem~\ref{thm:reduction}. In Section~\ref{sec:overview}, we sketch the \emph{sidestepping scheme} of the adversary as well as an idealized analysis of the scheme. We discuss a few challenges towards pinning down the optimal calibration error following our approach along with some other open problems in Section~\ref{sec:discuss}. Finally, in Section~\ref{sec:proof}, we present the scheme formally and then prove Theorem~\ref{thm:reduction}.

\section{Detour: The Sign-Preservation Game}\label{sec:sign}
As a detour, we introduce the following two-player sequential game called \game{}. We name these two players ``player A'' and ``player F'' to emphasize that they are analogous to the adversary and the forecaster in the prediction setting. An instance of \game{} with parameters $k$ and $r$, denoted by $\game(k, r)$, proceeds as follows. At the beginning of the game, there are $k$ empty cells numbered $1, 2, \ldots, k$. The game consists of at most $r$ rounds, and in each round:
\begin{enumerate}
    \item Player A may terminate the game immediately.
    \item Otherwise, player A chooses an empty cell with number $j \in [k]$.
    \item After knowing the value of $j$, player F places a \emph{sign} (either ``$+$'' or ``$-$'') into cell $j$, and cell $j$ is no longer empty.
\end{enumerate}

When the game ends, we examine the signs placed by player F. We say that the sign in cell $j$ is \emph{removed} if either one of the following two holds:
\begin{itemize}
    \item The sign is ``$+$'', and there exists $j' < j$ such that another sign is put into cell $j'$ after this sign is put into cell $j$.
    \item The sign is ``$-$'', and there exists $j' > j$ such that another sign is put into cell $j'$ after this sign is put into cell $j$.
\end{itemize}
If neither condition holds, the sign is said to be \emph{preserved}. Equivalently, a ``$+$'' sign (resp.\ ``$-$'' sign) is preserved if and only if all the subsequent signs are placed in cells with larger (resp.\ smaller) numbers. Player A's goal is to maximize the number of preserved signs, while player F tries to minimize this number.

\subsection{Connection to Binary Prediction}
Define the game value $\opt(k, r)$ as the maximum number of preserved signs in $\game(k, r)$, assuming that both players play optimally. We call a pair of numbers $(\alpha, \beta)$ \emph{admissible} if $\opt(k, k^\alpha)$ is lower bounded by $\Omega(k^\beta)$.

\begin{definition}\label{def:admissible}
    $(\alpha, \beta) \in (0, 1]^2$ is admissible if there exists constant $c_0 > 0$ such that $\opt(k, r) \ge c_0k^{\beta}$ holds for all integers $k \ge 1$ and $r \ge k^{\alpha}$.
\end{definition}

The following reduction theorem connects the \game{} game to the binary prediction setting. We will sketch the proof of Theorem~\ref{thm:reduction} in Section~\ref{sec:overview} and present the full proof in Section~\ref{sec:proof}.

\begin{theorem}\label{thm:reduction}
Suppose that $(\alpha, \beta)$ is an admissible pair. Let $c = \frac{2\beta + 1}{\alpha + 2\beta + 2}$. There exists a scheme of the adversary such that every forecaster incurs an expected calibration error of
\[\Ex{}{\calerr(T)} \ge \Omega(T^c/\sqrt{\log T}).\]
\end{theorem}

Note that the exponent $c = \frac{2\beta + 1}{\alpha + 2\beta + 2}$ is strictly greater than $1/2$ if and only if $\beta > \alpha / 2$. In the remainder of this section, we will prove the existence of such an admissible pair $(\alpha, \beta)$ with $\beta > \alpha / 2$ and then use it to prove Theorem~\ref{thm:main}.

\subsection{Lower Bounding the Game Value}
The following lemma gives two lower bounds on the optimal game value $\opt(\cdot, \cdot)$. The first states that Player A could make all signs preserved on $\game(k, \log k)$. The second states a ``tensorization'' property of the game, which allows us to lower bound a series of $\opt(\cdot, \cdot)$ given $\opt(k, r)$ for some specific $k$ and $r$.

\begin{lemma}\label{lem:opt-bound}
For any integer $t \ge 1$,
\begin{enumerate}
    \item $\opt(2^t - 1, t) = t$;
    \item $\opt(a, b) \ge c \ge 1$ implies $\opt(a^t, b^t) \ge \left(\frac{c+1}{2}\right)^t$.
\end{enumerate}
\end{lemma}

Another useful fact is the monotonicity of $\opt(k, r)$ in both $k$ and $r$.

\begin{lemma}\label{lem:opt-mono}
    For any $1 \le k_1 \le k_2$ and $1 \le r_1 \le r_2$, $\opt(k_1, r_1) \le \opt(k_2, r_2)$.
\end{lemma}

Lemmas \ref{lem:opt-bound}~and~\ref{lem:opt-mono} are proved in Appendix~\ref{sec:deferred-sign}.

\subsection{Proof of Theorem~\ref{thm:main}}

The first part of Lemma~\ref{lem:opt-bound} alone does not give any admissible pairs, because on an instance with $k$ cells, the number of preserved signs is at most $O(\log k) = o(k^{\beta})$ for any $\beta > 0$. However, when combined with the second part of Lemma~\ref{lem:opt-bound}, it indeed yields a non-trivial admissible pair, which in turn proves the lower bound in Theorem~\ref{thm:main}.

\begin{proof}[Proof of Theorem~\ref{thm:main}]
    Applying the first part of Lemma~\ref{lem:opt-bound} with $t = 8$ gives $\opt(255, 8) = 8$. Then, the second part of Lemma~\ref{lem:opt-bound}, together with the trivial case $\opt(1, 1) = 1$, implies that $\opt(255^t, 8^t) \ge (9/2)^t$ for any integer $t \ge 0$. Let $\alpha = \frac{\log 8}{\log 255}$ and $\beta = \frac{\log(9/2)}{\log 255}$. We will prove in the following that $(\alpha, \beta)$ is admissible. Then, Theorem~\ref{thm:main} would directly follow from Theorem~\ref{thm:reduction}.
    
    Fix $k \ge 1$, $r \ge k^{\alpha}$, and let $t = \lfloor\frac{\log k}{\log 255}\rfloor$. We have $k \ge 255^t$ and $t > \frac{\log k}{\log 255} - 1$. Furthermore, $r \ge k^{\alpha} \ge 255^{\alpha t} = 8^t$. By Lemma~\ref{lem:opt-mono}, we have
    \[
        \opt(k, r)
    \ge \opt(255^t, 8^t)
    \ge (9/2)^t
    > (9/2)^{\frac{\log k}{\log 255} - 1}
    =   \frac{2}{9}k^{\beta}.
    \]
    This shows that $(\alpha, \beta)$ is an admissible pair, and thus proves the theorem.
\end{proof}

\section{Overview of the Proof}\label{sec:overview}
In this section, we sketch a simplified version of the \emph{sidestepping scheme}, which will be used to prove Theorem~\ref{thm:reduction}. We will then explain how the \game{} game captures the essence of the scheme by drawing an analogy between the game and the sidestepping scheme. Finally, we present an idealized analysis that contains most of the key ideas behind the formal proof in Section~\ref{sec:proof}.

\subsection{A Sketch of the Scheme}
The sidestepping scheme is based on the notion of \emph{epochs}. The time horizon $1, 2, \ldots, T$ is divided into $k$ epochs of the same length $T/k$. The scheme chooses a probability $p^*_i$ at the beginning of the $i$-th epoch, and outputs $T/k$ independent samples from $\Ber(p^*_i)$ during this epoch. In the ideal case, we expect the forecaster to keep predicting a probability $\approx p^*_i$ throughout epoch $i$. Then, we would be able to lower bound $\calerr(T)$ by $k \cdot \sqrt{T/k} = \sqrt{Tk}$.

As discussed in Section~\ref{sec:obstacle}, this straightforward construction is vulnerable to untruthful forecasters whose predictions can be far away from the actual probability $p^*_i$. In particular, we need to prevent the forecaster from: (1) merging different epochs into a larger, coarse-grained bin, i.e., by predicting the average of $p^*_i$ at every time step; (2) covering up the errors made in the previous epochs. Resolving the first issue is relatively easier. Suppose that we choose the probabilities $p^*_1$ through $p^*_k$ to be $1/k, 2/k, \ldots, k/k$ and, in some epoch $i$, a significant fraction of the predictions are $(1/k)$-far from the actual probability $p^*_i$. Since each epoch has length $T/k$, these predictions lead to a calibration error of $(T/k)\cdot(1/k) = T/k^2$ in expectation. Then, we will be able to catch this error using the early stopping technique (Proposition~\ref{prop:early-stop}).

Otherwise, suppose that most of the forecaster's predictions are $(1/k)$-close to the true probability $p^*_i$ over epoch $i$. For simplicity, we assume for now that all the $T/k$ predictions are exactly $p^*_i$. Then, standard tail bounds for the binomial distribution imply that we expect an error of $|\Delta_{p^*_i}| \ge \Omega(\sqrt{T/k})$ after epoch $i$. Thus, summing over the $k$ epochs and taking a minimum with $T/k^2$ (the error when the forecaster is untruthful) seems to suggest an lower bound of $\min(T/k^2, \sqrt{Tk})$.

The issue with the above argument is that the forecaster might be able to cover up its error in later epochs, unless the probabilities of the future epochs are all in the right direction. For example, if $\Delta_{p^*_i} > 0$ at the end of epoch $i$, and the probability chosen for the next epoch satisfies $p^*_{i+1} < p^*_i$, the forecaster can decrease $\Delta_{p^*_i}$ by keeping predicting $p^*_i$ in epoch $i+1$, until $\Delta_{p^*_i}$ becomes close to zero. Fortunately, this kind of ``cover-up'' would not be possible if we chose the probabilities such that $p^*_{i'} > p^*_i$ for every $i' = i + 1, i + 2, \ldots, k$. This observation motivates us to choose $p^*_1, p^*_2, \ldots, p^*_k$ more carefully, so that the number of epochs whose $|\Delta_{p^*_i}|$ are preserved at the end of the scheme is maximized.

\subsection{Analogy between Binary Prediction and Sign-Preservation}
The above discussion closely resembles the \game{} game that we defined. In the scheme sketched above, we had $k$ possible choices, $1/k$ through $k/k$, for each $p^*_i$, and the $j$-th largest probability $j/k$ corresponds to the $j$-th cell in \game{}. The $i$-th epoch of the scheme is modeled by the $i$-th round of the game: (1) player A's action of choosing cell $j$ corresponds to the adversary's choice of $p^*_i = j/k$ for epoch $i$; (2) player F's action of placing a ``$+$''/``$-$'' sign can be thought of as getting $\Delta_{p^*_i} > 0$ or $\Delta_{p^*_i} < 0$ at the end of epoch $i$. Finally, a sign gets removed by another sign placed later (if the other sign is on the proper direction), since the error in epoch $i$ could be fixed by a later epoch $i'$, given that the sign of $p^*_{i'} - p^*_i$ is opposite to that of $\Delta_{p^*_i}$. Consequently, the number of epochs whose $|\Delta_{p^*_i}|$ are retained at the end of the scheme is modeled by the number of preserved signs at the end of the \game{} game.

\subsection{Proof Sketch of Theorem~\ref{thm:reduction}}
We sketch a proof of Theorem~\ref{thm:reduction} in the following. Let $(\alpha, \beta)$ be an admissible pair, and $k$ be a parameter to be determined later. In contrast to the scheme described above, we will divide the time horizon into $k^\alpha$ epochs instead, and each epoch has length $T/k^\alpha$.

The adversary simulates an instance of $\game(k, k^\alpha)$ where player A plays optimally. Every time player A chooses a cell with number $j$, the adversary chooses $p^*_i = j/k$ for the next epoch $i$. In other words, the next $T/k^{\alpha}$ bits will be independent samples from $\Ber(j/k)$. 

Within epoch $i$, we say that a prediction given by the forecaster is \emph{untruthful} if the predicted probability is $(1/k)$-far from $p^*_i$. Clearly, the forecaster has to make either $\Omega(T/k^\alpha)$ untruthful predictions, or $\Omega(T/k^\alpha)$ truthful ones. In the former case, we claim that each untruthful prediction increases the cumulative calibration error by an $\Omega(1/k)$ amount, so the total increase throughout this epoch will be at least $\Omega(T/k^{\alpha+1})$. Then, the adversary would be able to catch this $\Omega(T/k^{\alpha+1})$ error using the ``early stopping'' trick.

Otherwise, suppose that epoch $i$ is truthful. Then, lower bounds on binomial tails imply that there will be an $\Omega(\sqrt{T/k^{\alpha}})$ error in expectation after this epoch. Furthermore, this error cannot be significantly reduced if no later epochs is assigned a probability $p^*_{i'}$ with $(p^*_{i'} - p^*_i)\cdot\Delta_{p^*_i} < 0$. Thus, if we choose (on behalf of player F in \game{}) the sign for this cell as the sign of $\Delta_{p^*_i}$, $|\Delta_{p^*_i}|$ will still be $\Omega(\sqrt{T/k^{\alpha}})$ at the end of the scheme as long as the sign placed in this round is preserved at the end of the \game{} game.

Since $(\alpha, \beta)$ is admissible, there will be at least $\Omega(k^{\beta})$ preserved signs, thus giving a lower bound of $\Omega(k^{\beta}\sqrt{T/k^{\alpha}})$.
Taking a minimum with the $\Omega(T/k^{\alpha+1})$ error (in case of an epoch with too many untruthful predictions) and plugging in the optimal choice of $k = T^{\frac{1}{\alpha+2\beta+2}}$ gives the claimed $\Omega(T^c)$ lower bound for $c = \frac{2\beta + 1}{\alpha + 2\beta + 2}$.

\section{Discussion}\label{sec:discuss}
In this section, we discuss a few open directions for further understanding the optimal calibration error that can be achieved in the binary prediction setting.

\paragraph{Gap between upper and lower bounds.} In light of Theorem~\ref{thm:reduction}, an immediate open problem is to find other admissible pairs $(\alpha, \beta)$ that result in a larger exponent $c = \frac{2\beta + 1}{\alpha + 2\beta + 2}$ in the lower bound. In particular, the best possible exponent we can get from Theorem~\ref{thm:reduction} is $c = 3/5$ if $(1, 1)$ is admissible, i.e., $\Omega(k)$ signs can be preserved in a \game{} game with $k$ cells and $k$ rounds. Either proving or disproving this would help us to understand the limit of the approach based on the sidestepping scheme and Theorem~\ref{thm:reduction}.

Another natural open question is whether the $O(T^{2/3})$ upper bound is indeed optimal. In particular, can we construct a better forecaster by proving a converse of Theorem~\ref{thm:reduction} that translates upper bounds on $\opt(\cdot, \cdot)$ into actual strategies for the forecaster? While such a converse is likely to exist when the adversary is epoch-based (and even announces the probability of each epoch at the beginning of that epoch), extending this reduction in the converse direction to more general cases seems challenging.

\paragraph{The power of adaptivity.} Our proof of the lower bound is based on an adaptive scheme for the adversary. More exactly, the scheme uses adaptivity on two different levels: (1) The adversary decides the probability $p^*_i$ of an epoch $i$ based on the gameplay of a \game{} instance. In general, player A of \game{} is allowed to choose the cells adaptively based on the state of the game, which in turn means that $p^*_i$ are chosen adaptively; (2) When we formally prove Theorem~\ref{thm:reduction}, we will need to apply Proposition~\ref{prop:early-stop} to transform the scheme---which is only guaranteed to give a large $\calerr(t)$ at some point $t$---into another scheme with a large $\Ex{}{\calerr(T)}$, and the transformation based on the early stopping trick is inherently adaptive.

Nevertheless, we conjecture that both uses of adaptivity can be replaced by randomization: (1) The lower bounds on $\opt(\cdot, \cdot)$ in Lemma~\ref{lem:opt-bound} can still be achieved (up to a constant factor) in expectation by a non-adaptive yet randomized strategy for player A; (2) The adaptive early stopping strategy in the proof of Proposition~\ref{prop:early-stop} can also be replaced by a randomized one, e.g., that chooses the number of epochs uniformly at random from $1, 2, \ldots, k^\alpha$. Thus, as far as super-$\sqrt{T}$ lower bounds are concerned, adaptivity appears inessential to the adversary, though adaptivity does greatly simplify the analysis of the scheme. Furthermore, it remains an interesting yet challenging open problem to understand whether the extra power brought by the adaptivity increases the calibration error that the optimal forecaster has to incur.

\section{Proof of Theorem~\ref{thm:reduction}}\label{sec:proof}
\subsection{The Sidestepping Scheme}
We formally define the sidestepping scheme in Algorithm~\ref{algo:scheme} and the epochs in the scheme are defined in Algorithm~\ref{algo:epoch}. The core of the scheme is to simulate an instance of the game $\game(k, k^{\alpha})$ for some carefully chosen $k$. In this simulated game, player A plays the optimal strategy while the sidestepping scheme, perhaps paradoxically, plays on behalf of player F. This situation can be best illustrated as Figure~\ref{fig:MIM}, where the sidestepping scheme plays as a ``man-in-the-middle'' and connects an optimal player A for \game{} to the forecaster in the prediction setting.

\begin{figure}[H]
    \centering
    \includegraphics[scale=0.25]{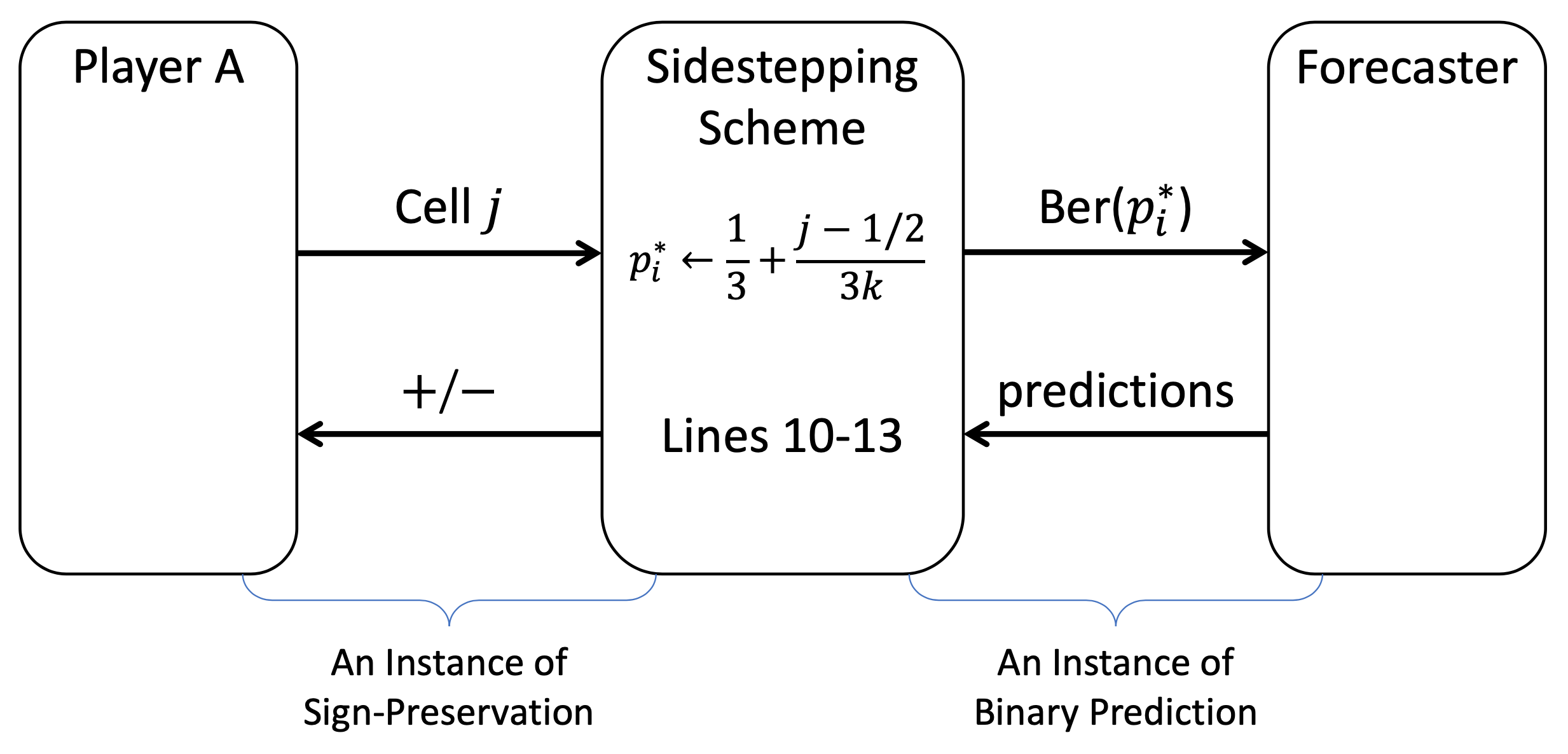}
    \caption{The sidestepping scheme working as a man-in-the-middle. Player A is playing \game{} and the forecaster is in the binary prediction setting from their perspectives.}
    \label{fig:MIM}
\end{figure}

\begin{algorithm2e}[H]\label{algo:scheme}
    \caption{Sidestepping Scheme} \KwIn{Horizon length $T$ and parameters $\alpha, \beta \in (0, 1]$.}
    $k \gets T^{1/(\alpha + 2\beta + 2)}$;
    $\theta \gets \frac{1}{1440}\sqrt{\frac{T}{k^{\alpha}\ln T}}$\;
    Simulate an instance of $\game(k, k^{\alpha})$\;
    \For{$i = 1, 2, \ldots, k^{\alpha}$} {
        \uIf{player A terminates the game in round $i$} {
            \textbf{break}\;
        }
        Let $j \in [k]$ be the cell chosen by player A in round $i$\;
        $\Int_i \gets$ interval $(\frac{1}{3} + \frac{j - 1}{3k}, \frac{1}{3} + \frac{j}{3k})$\;
        $p^*_i \gets \frac{1}{3} + \frac{j-1/2}{3k}$ ; \tcp{the middle point of $\Int_i$}
        Call $\Epoch(T/k^{\alpha}, \Int_i, p^*_i, \theta)$\;
        \uIf{$\sum_{p \in P \cap \Int_i}\Delta^{+}_p \ge \sum_{p \in P \cap \Int_i}\Delta^{-}_p$} {
            Let player F place ``$+$'' into cell $j$\;    
        } \uElse {
            Let player F place ``$-$'' into cell $j$\;
        }
    }
\end{algorithm2e}

\begin{algorithm2e}[H]\label{algo:epoch}
    \caption{$\Epoch(m, \Int, p^*, \theta)$}
    \For{$i = 1, 2, \ldots, m$} {
        \uIf{$\sum_{p \in P\cap\Int}|\Delta_p| \ge \theta$} {
            \textbf{break}\;
        }
        Draw $b \sim \Ber(p^*)$\;
        Output bit $b$\;
    }
\end{algorithm2e}

The scheme differs from the simplified version in Section~\ref{sec:overview} in the following two aspects. First, the probability $p^*_i$ is restricted to the interval $[1/3, 2/3]$ (instead of $[0, 1]$), so that the binomial distribution with parameter $p^*_i$ would have a tail that is lower bounded by Gaussian tails. More specifically, when player A chooses some cell $j$ in the game, we start an epoch associated with probability $p^*_i = \frac{1}{3} + \frac{j - 1/2}{3k}$. Note that $p^*_i$ is exactly the middle point of the $j$-th interval when $[1/3, 2/3]$ is partitioned into $k$ intervals of length $1/(3k)$.

Second, an epoch may span less than $T/k^\alpha$ time steps. In particular, we set a threshold $\theta$ and end an epoch as soon as the interval $\Int$ associated with the epoch already contributes at least $\theta$ to the cumulative error. The purpose of this slight change is mostly to simplify the analysis. As a result, the sidestepping scheme may end before $T$ time steps. In the following, we use random variable $\Tact$ to denote the number of time steps that the scheme actually lasts. At the end of the proof, we will transform the sidestepping scheme into another scheme that spans exactly $T$ steps using Proposition~\ref{prop:early-stop}.

\subsection{Classification of Epochs}
For each possible execution of the sidestepping scheme (Algorithm~\ref{algo:scheme}), we say that an epoch is \emph{untruthful} if the forecaster makes too many predictions that are $\Omega(1/k)$-far away from the actual probability; otherwise it is said to be \emph{truthful}.
\begin{definition}[Untruthful epochs]\label{def:truthful}
    An epoch $i$ associated with interval $\Int_i$ is untruthful if, within epoch $i$, the forecaster makes at least $T/(2k^{\alpha})$ predictions with values outside $\Int_i$.
\end{definition}

We call a truthful epoch \emph{negligible} if, when the epoch ends, the interval associated with it contributes less than $\theta = \frac{1}{1440}\sqrt{\frac{T}{k^{\alpha}\ln T}}$ to the cumulative calibration error at that time; otherwise the epoch is said to be \emph{non-negligible}. By our definition of $\Epoch$ (Algorithm~\ref{algo:epoch}), an epoch is negligible only if it takes exactly $m$ time steps.

\begin{definition}[Negligible epochs]\label{def:negligible}
    A truthful epoch $i$ associated with interval $\Int_i$ is negligible if, when epoch $i$ ends at time step $t$, it holds that $\sum_{p\in P\cap\Int_i}|\Delta_p(t)| < \theta$.
\end{definition}

Finally, for a truthful and non-negligible epoch, we call it \emph{covered} if, at the end of the scheme, its contribution to $\calerr(\Tact)$ is less than $\theta / 4$; otherwise we call it \emph{uncovered}. In other words, an epoch is covered if the predictions in later epochs cover up a significant portion of the mistakes made by the forecaster in the epoch.

\begin{definition}[Covered epochs]\label{def:covered}
    A truthful and non-negligible epoch $i$ associated with interval $\Int_i$ is covered if, when the scheme ends after $\Tact$ time steps, it holds that $\sum_{p\in P\cap\Int_i}|\Delta_p(\Tact)| < \theta / 4$.
\end{definition}

Pictorially, the relation between different classes of epochs defined in Definitions \ref{def:truthful}~through~\ref{def:covered} is demonstrated in Figure~\ref{fig:epochs}.

\begin{figure}[H]
    \centering
    \includegraphics[scale=0.35]{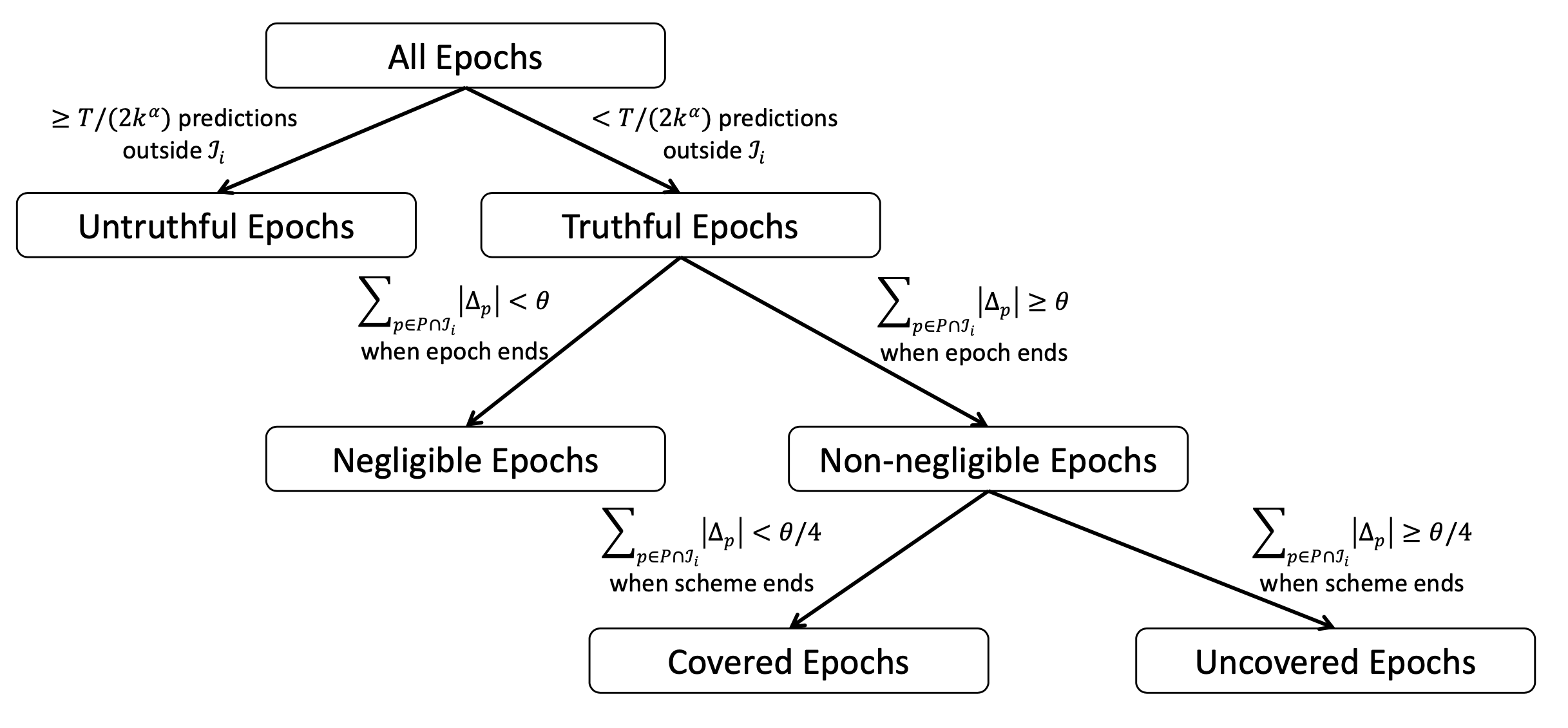}
    \caption{The relation between different classes of epochs.}
    \label{fig:epochs}
\end{figure}

\subsection{Auxiliary Lemmas}
Lemmas \ref{lem:truthful}, \ref{lem:negligible} and \ref{lem:covered} state that the following three hold with high probability: (1) either every epoch is truthful, or $\calerr(t)$ is large at some point $t$; (2) every truthful epoch is non-negligible; (3) for every non-negligible epoch $i$, if the sign placed in round $i$ of \game{} is preserved, epoch $i$ is uncovered.

\begin{lemma}\label{lem:truthful}
    Fix $i \in [k^{\alpha}]$ and let $B \coloneqq \frac{T}{48k^{\alpha + 1}}$. Suppose that epoch $i$ spans the time steps $t_0 + 1, t_0 + 2, \ldots, t_1$. The probability that epoch $i$ is untruthful and $\calerr(t) < B$ holds for every $t = t_0, t_0 + 1, \ldots, t_1$ is at most $\exp(-\Omega(T/k^{\alpha + 2})) = o(1/T)$.
\end{lemma}

\begin{lemma}\label{lem:negligible}
    For any fixed $i \in [k^{\alpha}]$, the probability that epoch $i$ is truthful and negligible is at most $T^{-2} = o(1/T)$.
\end{lemma}

\begin{lemma}\label{lem:covered}
    For any fixed $i \in [k^{\alpha}]$, the probability that the following two hold simultaneously is at most $T \cdot \exp(-\theta/(12k)) = o(1/T)$:
    (1) epoch $i$ is truthful, non-negligible, and covered; (2) the sign placed in the $i$-th round of the \game{} game is preserved.
\end{lemma}

 All these three lemmas are proved by applying standard concentration and anti-concentration bounds to carefully chosen quantities tailored to the epoch in question. The proofs are deferred to Appendix~\ref{sec:deferred-proof}.

\subsection{Putting Everything Together}
Now we are ready to prove Theorem~\ref{thm:reduction}.

\begin{proof}[Proof of Theorem \ref{thm:reduction}]
Let $k = T^{1/(\alpha + 2\beta + 2)}$ and $\theta = \frac{1}{1440}\sqrt{\frac{T}{k^{\alpha}\ln T}}$ as in Algorithm~\ref{algo:scheme}, and $c = \frac{2\beta+1}{\alpha+2\beta+2}$ as in the statement of the theorem. Define $B\coloneqq \min\left(T/(48k^{\alpha+1}), c_0\theta k^{\beta}/4\right) = \Omega\left(T^c/\sqrt{\log T}\right)$, where $c_0$ is the constant for the admissible pair $(\alpha, \beta)$ in Definition~\ref{def:admissible}.  Recall that $\Tact \le T$ denotes the number of time steps that the sidestepping scheme actually lasts, and $\maxerr(\Tact)$ denotes $\max_{t\in [\Tact]}\calerr(t)$. We will show that the sidestepping scheme defined as Algorithm~\ref{algo:scheme}, when running against any forecaster, satisfies that $\pr{\maxerr(\Tact) \ge B} \ge 1 - o(1)$.

To upper bound $\pr{\maxerr(\Tact) < B}$, we define $\Etruth$ as the event that all the $k^{\alpha}$ epochs are truthful (in the sense of Definition~\ref{def:truthful}). Then, we note that
\begin{align*}
    &\pr{\maxerr(\Tact) < B}\\
\le~~&\sum_{i=1}^{k^{\alpha}}\pr{\text{epoch }i\text{ is untruthful}\wedge \maxerr(\Tact) < B}\\
&~~+ \pr{\Etruth \wedge \maxerr(\Tact) < B} \tag{union bound}\\
\le~~&\sum_{i=1}^{k^{\alpha}}\pr{\text{epoch }i\text{ is untruthful} \wedge \maxerr(\Tact) < B}\\
&~~+ \pr{\Etruth \wedge \calerr(\Tact) < B}. \tag{$\maxerr(\Tact) < B \implies \calerr(\Tact) < B$}
\end{align*}
By Lemma~\ref{lem:truthful}, each term in the summation is at most $o(1/T)$, so the whole summation is upper bounded by $k^{\alpha} \cdot o(1/T) = o(1)$. It remains to prove $\pr{\Etruth \wedge \calerr(\Tact) < B} = o(1)$.

Let $\Enegl$ denote the event that at least one of the $k^{\alpha}$ epochs is truthful and negligible, and $\Ecover$ be the event that there exists $i\in[k^\alpha]$ such that: (1) epoch $i$ is truthful, non-negligible and covered; (2) the sign placed in the $i$-th round of \game{} is preserved in the end. Then, by Lemmas \ref{lem:negligible}~and~\ref{lem:covered} and a union bound over the $k^{\alpha} \le T$ epochs, $\pr{\Etruth\wedge(\Enegl \vee \Ecover)} = o(1)$. We will show in the following that the event $\Etruth \wedge \calerr(\Tact) < B$ is a subset of $\Etruth\wedge(\Enegl \vee \Ecover)$, and thus it holds that
\[
    \pr{\Etruth \wedge \calerr(\Tact) < B}
\le \pr{\Etruth\wedge(\Enegl \vee \Ecover)} = o(1).
\]
We will prove the contrapositive: assuming that event $\Etruth$ happens yet neither $\Enegl$ nor $\Ecover$ happens, it holds that $\calerr(\Tact) \ge B$. Let $S \subseteq [k^{\alpha}]$ denote the set of indices $i$ such that the sign placed in round $i$ of the \game{} game is preserved. Since $(\alpha, \beta)$ is admissible with constant $c_0$, $|S| \ge c_0k^{\beta}$. For each $i \in S$, since we assumed $\Etruth\wedge \overline{\Enegl} \wedge \overline{\Ecover}$, epoch $i$ is truthful, non-negligible and uncovered. Then, by definition, it holds that
$\sum_{p \in P\cap\Int_i}|\Delta_p(\Tact)|
\ge \theta / 4$. Since the intervals $\Int_i$ are disjoint for different indices $i \in S$, we have
\[
    \calerr(\Tact)
\ge \sum_{i \in S}\sum_{p \in P\cap\Int_i}|\Delta_p(\Tact)|
\ge c_0k^{\beta} \cdot (\theta / 4)
\ge B.
\]

This completes the proof of $\pr{\maxerr(\Tact) \ge B} \ge 1 - o(1)$ when running Algorithm~\ref{algo:scheme} against any forecaster. For sufficiently large $T$, the $o(1)$ term is at most $1/2$.
Then, by Proposition~\ref{prop:early-stop}, there exists a scheme such that $\Pr[\calerr(T) \ge B / 2] \ge 1/2$, which implies the lower bound $\Ex{}{\calerr(T)} \ge B / 4 = \tilde\Omega(T^c)$.
\end{proof}

\nocite{*}

\bibliographystyle{alpha}
\bibliography{main}

\appendix

\section{Deferred Proofs from Section~\ref{sec:sign}}\label{sec:deferred-sign}
\begin{proof}[Proof of Lemma~\ref{lem:opt-bound}]
The first part follows from a simple strategy that resembles a binary search: Player A chooses cell $2^{t-1}$ in the first round. If the sign placed by player F is ``$+$'', proceed with the remaining $t-1$ rounds on the $2^{t-1}-1$ cells with numbers $2^{t-1} + 1, 2^{t-1} + 2, \ldots, 2^t - 1$; otherwise, proceed with cells $1, 2, \ldots, 2^{t-1} - 1$. Then, none of the $t$ signs will be removed in the end, and this proves $\opt(2^t - 1, t) = t$.

We prove the second part by induction. The inequality clearly holds for $t = 1$ since the assumption implies $\opt(a, b) \ge c \ge (c + 1)/2$. For $t \ge 2$, we consider the following strategy for $\game(a^t, b^t)$: Player A divides the $a^t$ cells into $a$ ``super-cells'', each corresponding to $a^{t-1}$ contiguous cells. Then, Player A simulates a hypothetical instance of $\game(a, b)$, denoted by $\SPouter$, in the following sense: When one of the $a$ super-cells is chosen in round $i$ of $\SPouter$, Player A simulates an actual instance of $\game(a^{t-1}, b^{t-1})$, denoted by $\SPinner_i$, on the $a^{t-1}$ cells corresponding to that super-cell, i.e., whenever a cell is chosen in $\SPinner_i$, Player A chooses the corresponding cell in the actual $\game(a^t, b^t)$ instance. After $\SPinner_i$ terminates, Player A, on behalf of the ``Player F'' in $\SPouter$, places a sign into that super-cell according to the majority of the preserved signs in $\SPinner_i$.

Now we count the preserved signs in the original game $\game(a^t, b^t)$. If the sign in a super-cell is preserved at the end of $\SPouter$, any preserved sign in the $a^{t-1}$ corresponding cells (at the end of the corresponding $\SPinner$ instance) that agrees with the sign in that super-cell will also be preserved at the end of $\game(a^t, b^t)$. By our choice of the sign's direction, there are at least $\opt(a^{t-1}, b^{t-1})/2$ such signs for each preserved sign in $\SPouter$. Moreover, for the preserved sign that is placed in the last round of $\SPouter$, all the $\ge \opt(a^{t-1}, b^{t-1})$ remaining signs in the corresponding super-cell will be preserved. By the inductive hypothesis that $\opt(a^{t-1}, b^{t-1}) \ge \left(\frac{c+1}{2}\right)^{t-1}$, we have
\[
    \opt(a^t, b^t)
\ge (c-1)\cdot\frac{\opt(a^{t-1}, b^{t-1})}{2} + \opt(a^{t-1}, b^{t-1}) \ge \left(\frac{c+1}{2}\right)^t,
\]
which completes the induction.
\end{proof}

\begin{proof}[Proof of Lemma~\ref{lem:opt-mono}]
    On an instance of $\game(k_2, r_2)$, Player A can simulate the optimal strategy for the game $\game(k_1, r_1)$ on cells $1, 2, \ldots, r_1$. Player A ends the game when the $\game(k_1, r_1)$ instance terminates. Since $r_1 \le r_2$, the simulated game never lasts more than $r_2$ rounds.
    By definition, there will be at least $\opt(k_1, r_1)$ preserved signs, so we have $\opt(k_2, r_2) \ge \opt(k_1, r_1)$.
\end{proof}

\section{Deferred Proofs from Section~\ref{sec:proof}}\label{sec:deferred-proof}

The following concentration bound is an immediate corollary of the Azuma-Hoeffding inequality for submartingales.

\begin{lemma}\label{lem:azuma}
    Suppose that random variables $X_1, X_2, \ldots, X_m$ satisfy that for every $t \in [m]$: (1) $X_t \in [-1, 1]$ almost surely; (2) $\Ex{}{X_t|X_1, X_2, \ldots, X_{t-1}} \ge \mu$. Then, for any $c < m\mu$,
        \[\pr{\sum_{t=1}^mX_t \le c} \le \exp\left(-\frac{(m\mu - c)^2}{2m}\right).\]
\end{lemma}

The following anti-concentration bound for binomial distributions follows from the Berry-Esseen theorem.

\begin{lemma}\label{lem:berry}
    Suppose that $p \in [1/3, 2/3]$, $Z$ follows the binomial distribution $B(m, p)$, and $g$ follows the standard Gaussian distribution $N(0, 1)$. For any $c \in \mathbb{R}$, it holds that
        \[\left|\pr{\frac{Z - mp}{\sqrt{mp(1-p)}} \ge c} - \pr{g \ge c}\right| \le O\left(\frac{1}{\sqrt{m}}\right),\]
    where the $O(\cdot)$ notation hides a universal constant that does not depend on $m$ or $p$.
\end{lemma}

Now we are ready to prove Lemmas \ref{lem:truthful}~through~\ref{lem:covered}.

\begin{proof}[Proof of Lemma~\ref{lem:truthful}]
Let $\Int_i = (l_i, r_i)$ be the interval associated with epoch $i$ and $p^*_i = (l_i + r_i) / 2$ be its middle point. By the choice of $\Int_i$ and $p^*_i$ in Algorithm~\ref{algo:scheme}, $p^*_i - l_i = r_i - p^*_i = \frac{1}{6k}$. Define
\[\hat\Delta(t) \coloneqq \sum_{p\in P \cap [0, l_i]}\Delta_p(t) + \sum_{p\in P \cap [r_i, 1]}[-\Delta_p(t)]\]
as a proxy of $\calerr(t)$. It can be easily verified that both $\hat\Delta(t) \le \calerr(t)$ and $-\hat\Delta(t) \le \calerr(t)$.

Clearly, a prediction with value inside $\Int_i$ does not change the value of $\hat\Delta$; in contrast, whenever a probability outside $\Int_i$ is predicted, $\hat\Delta$ is incremented by at least $\Omega(1/k)$ in expectation. To see this, suppose that the forecaster predicts $p(t) \le l_i$ at time step $t$. Then, the expected increment in $\hat\Delta$ is given by
\begin{align*}
    \Ex{}{\Delta_{p(t)}(t) - \Delta_{p(t)}(t - 1)}
&=   \Ex{}{m_{p(t)}(t) - m_{p(t)}(t - 1)} - p(t) \cdot \Ex{}{n_{p(t)}(t) - n_{p(t)}(t - 1)}\\
&=   p^*_i - p(t)
\ge p^*_i - l_i = \frac{1}{6k}.
\end{align*}
Moreover, the increment in $\hat\Delta$ is always bounded between $-1$ and $1$. Similarly, whenever a prediction $p(t) \ge r_i$ is made, the increment in $\hat\Delta$ is always between $-1$ and $1$ and has expectation $p(t) - p^*_i \ge r_i - p^*_i = \frac{1}{6k}$.

Let $m \coloneqq T/(2k^{\alpha})$.
Assuming that epoch $i$ is untruthful, there exists a unique time step $t_2 \in [t_0, t_1]$ when the forecaster makes the $m$-th prediction that falls outside $\Int_i$. We will prove that $\hat\Delta(t_2) - \hat\Delta(t_0) \ge \frac{m}{12k}$ with high probability, which implies that either $\hat\Delta(t_2) \ge \frac{m}{24k}$ or $\hat\Delta(t_0) \le -\frac{m}{24k}$. Then, we would have $\max(\calerr(t_0), \calerr(t_2)) \ge \frac{m}{24k} = B$ as desired. Indeed, our discussion above indicates that $\hat\Delta(t_2) - \hat\Delta(t_0)$ can be written as a sum of $m$ random variables $X_1, X_2, \ldots, X_m$ satisfying that for each $j \in [m]$: (1) $X_j \in [-1, 1]$ almost surely; (2) $\Ex{}{X_j|X_1, X_2, \ldots, X_{j-1}} \ge \frac{1}{6k}$. Then, by the Azuma-Hoeffding inequality (in the form of Lemma~\ref{lem:azuma}), it holds that
\[
    \pr{\hat\Delta(t_2) - \hat\Delta(t_0) \le \frac{m}{12k}} \le \exp\left(-\frac{m}{288k^2}\right)
=   \exp\left(-\Omega\left(\frac{T}{k^{\alpha + 2}}\right)\right).
\]
Finally, since Algorithm~\ref{algo:scheme} chooses $k = T^{\frac{1}{\alpha + 2\beta + 2}}$ and requires $\beta > 0$, we have $\frac{T}{k^{\alpha + 2}} = T^{\frac{2\beta}{\alpha + 2\beta + 2}} = \Omega(\poly(T))$. Thus, $\exp(-\Omega(T/k^{\alpha+2})) = o(1/T)$, which completes the proof.
\end{proof}

\begin{proof}[Proof of Lemma~\ref{lem:negligible}]
Let $\Int_i = (l_i, r_i)$ be the interval associated with epoch $i$. For epoch $i$ to be truthful but negligible, the forecaster needs to make at least $T/(2k^{\alpha})$ predictions with values inside $\Int_i$. Let $m = T/(324k^{\alpha}\ln T)$. We may further decompose epoch $i$ into $162\ln T$ blocks, each with at least $m$ predictions that fall into $\Int_i$. To prove the lemma, it suffices to show that, conditioning on the bits and predictions before each block, the probability that $\sum_{p\in P\cap\Int_i}|\Delta_p|$ remains less than $\theta$ throughout the block is at most $1 - 3^{-4}$. Assuming this, the probability that epoch $i$ becomes negligible after the $162\ln T$ blocks is at most $(1 - 3^{-4})^{162\ln T} \le e^{-2\ln T} = T^{-2}$, as claimed by the lemma.

Fix a block with $\ge m$ predictions inside $\Int_i$ and let $t_0$ be the time step before the start of the block. Let $\delta \coloneqq 1/(10\sqrt{m})$ and $p^*_i$ be the middle point of $\Int_i$. We divide the prediction values into the following four groups:
\begin{itemize}
    \item $P_1 \coloneqq \{p \in P \cap \Int_i: \Delta_p(t_0) \ge 0, p^*_i- p \ge -\delta)\}$.
    \item $P_2 \coloneqq \{p \in P \cap \Int_i: \Delta_p(t_0) \ge 0, p^*_i- p < -\delta\}$.
    \item $P_3 \coloneqq \{p \in P \cap \Int_i: \Delta_p(t_0) < 0, p^*_i- p \le \delta\}$.
    \item $P_4 \coloneqq \{p \in P \cap \Int_i: \Delta_p(t_0) < 0, p^*_i- p > \delta\}$.
\end{itemize}
Note that each prediction at probability $p$ increases $\Delta_p$ by $p^*_i\cdot(1 - p) + (1 - p^*_i)\cdot(0 - p) = p^*_i- p$ in expectation. Thus, the above definition basically says that each prediction that falls into $P_1 \cup P_3$ will either push $\Delta_p$ away from $0$, or push it towards $0$ by at most $\delta = O(1/\sqrt{m})$. In contrast, each prediction in $P_2 \cup P_4$ will push $\Delta_p$ in the opposite direction of the sign of $\Delta_p(t_0)$ by at least $\delta = \Omega(1/\sqrt{m})$.

Since there are $m$ predictions inside $\Int_i$ within the block, at least one of the four sets $P_1$ through $P_4$ receives at least $m / 4$ predictions. In the remainder of the proof, we will prove the following claim: for each $P_j$, with probability at least $1/3$, $\sum_{p\in P_j}|\Delta_p|$ will reach $\theta$ before or exactly when the forecaster makes the $(m/4)$-th prediction with value inside $P_j$. Thus, the scheme would have terminated before that. Assuming this, we may pretend that there are four independent bit sequences, each consisting of independent samples from $\Ber(p^*_i)$. When a probability inside $P_j$ is predicted, the bit output by the adversary actually comes from the $j$-th sequence. Since the four bit sequences are independent, $\sum_{p\in P\cap\Int_i}|\Delta_p|$ will reach $\theta$ at some point in this block with probability at least $3^{-4}$. This would then prove the lemma.

\paragraph{Proofs for $P_1$ and $P_3$.}
The proofs for $P_1$ and $P_3$ are symmetric, so we only present the proof for $P_1$ in the following. Define the quantity $\hat\Delta(t) \coloneqq \sum_{p \in P_1}\Delta_p(t)$. Whenever a value $p \in P_1$ is predicted, $\hat\Delta$ is incremented by $b - p$, where $b \sim \Ber(p^*)$. Moreover, by definition of $P_1$, every $p \in P_1$ is upper bounded by $p^*_i + \delta$. Thus, the increment of $\hat\Delta$ after the first $m/4$ predictions in $P_1$ is lower bounded by $Z - (p^*_i + \delta)\cdot(m/4)$, where $Z$ follows the binomial distribution $B(m/4, p^*_i)$. Let $g \sim N(0, 1)$ be a standard Gaussian random variable. By Lemma~\ref{lem:berry},
\begin{align*}
    \pr{Z \ge mp^*_i/4 + m\delta/2}
&=   \pr{\frac{Z - mp^*_i/4}{\sqrt{mp^*_i(1-p^*_i)/4}} \ge \frac{m\delta/2}{\sqrt{mp^*_i(1-p^*_i)/4}}}\\
&\ge \pr{\frac{Z - mp^*_i/4}{\sqrt{mp^*_i(1-p^*_i)/4}} \ge \frac{3}{10\sqrt{2}}} \tag{$\delta = \frac{1}{10\sqrt{m}}$, $p^*_i(1-p^*_i)\ge\frac{2}{9}$}\\
&\ge \pr{g \ge \frac{3}{10\sqrt{2}}} - O(1/\sqrt{m/4}) \tag{Lemma~\ref{lem:berry}}\\
&\ge 0.416 - O(1/\sqrt{m}) \tag{CDF of Gaussian}\\
&\ge \frac{1}{3}. \tag{for sufficiently large $m$}
\end{align*}
Thus, with probability at least $1/3$, $\hat\Delta$ increases by at least $(mp^*_i/4 + m\delta/2) - (p^*_i+\delta)\cdot(m/4) = m\delta/4$ from time $t_0$ to $t_1$. This implies that either $\hat\Delta(t_0) \le -m\delta/8$ or $\hat\Delta(t_1) \ge m\delta/8$. Note that both $\hat\Delta$ and $-\hat\Delta$ are lower bounds on $\sum_{p\in P\cap\Int_i}|\Delta_p|$. So $\sum_{p\in P\cap\Int_i}|\Delta_p|$ must reach $m\delta/8 = \sqrt{m}/80 = \frac{1}{1440}\sqrt{\frac{T}{k^\alpha\ln T}} = \theta$ at some point, and the epoch should have been terminated.

\paragraph{Proofs for $P_2$ and $P_4$.} Again, we only present the proof for $P_2$ and the proof for $P_4$ is symmetric. Define $\hat\Delta(t) \coloneqq \sum_{p \in P_2}[-\Delta_p(t)]$. When a probability $p \in P_2$ is predicted by the forecaster, $\hat\Delta$ is incremented by $p - b$, where $b \sim \Ber(p^*)$ and $p > p^*_i + \delta$. Therefore, the total increment in $\hat\Delta$ after the $m/4$ predictions is at least $(m/4)(p^*+\delta) - Z$, where $Z \sim B(m/4, p^*)$. Again, applying Lemma~\ref{lem:berry} gives $\pr{Z \le mp^*_i/4} \ge 1/3$ for sufficiently large $T$. When $Z \le mp^*_i/4$, the increment would be at least $m\delta/4$, which further implies that $\sum_{p\in P\cap\Int_i}|\Delta_p|$ should have reached $\theta = m\delta/8$ and the epoch should have terminated.
\end{proof}

\begin{proof}[Proof of Lemma~\ref{lem:covered}]
Let $\Int_i = (l_i, r_i)$ be the interval associated with epoch $i$ and $p^*_i = (l_i + r_i) / 2$ be its middle point. Let $t_0$ be the last time step of epoch $i$. If epoch $i$ is truthful and non-negligible, it holds that $\sum_{p\in P\cap\Int_i}|\Delta_p(t_0)| \ge \theta$. 
We consider the following two cases depending on whether $\sum_{p\in P\cap\Int_i}\Delta^{+}_p(t_0)$ or $\sum_{p\in P\cap\Int_i}\Delta^{-}_p(t_0)$ is larger.


\paragraph{Case 1: $\sum_{p\in P\cap\Int_i}\Delta^{+}_p(t_0) \ge \sum_{p\in P\cap\Int_i}\Delta^{-}_p(t_0)$.} As the two summations sum up to $\sum_{p\in P\cap\Int_i}|\Delta_p(t_0)| \ge \theta$, we have $\sum_{p\in P\cap\Int_i}\Delta^{+}_p(t_0) \ge \theta / 2$. Let $P^{+}\coloneqq \{p\in P \cap \Int_i: \Delta_p(t_0) \ge 0\}$. For each time step $t$, define the quantity $\hat\Delta(t)$ as
    \[\hat\Delta(t) \coloneqq \sum_{p\in P^{+}}\Delta_p(t).\]
Then, we have $\hat\Delta(t_0) = \sum_{p\in P^{+}}\Delta_p(t_0) = \sum_{p\in P\cap\Int_i}\Delta^{+}_p(t_0) \ge \theta / 2$.

For epoch $i$ to be covered, when the scheme terminates at time $\Tact$, it should hold that $\sum_{p\in P\cap\Int_i}|\Delta_p(\Tact)| < \theta / 4$. Since $P^{+}$ is a subset of $P\cap\Int_i$ and $\Delta_p \le |\Delta_p|$, it should also hold that $\hat\Delta(\Tact) = \sum_{p\in P^{+}}\Delta_p(\Tact) \le \sum_{p\in P\cap\Int_i}|\Delta_p(\Tact)| < \theta / 4$. Thus, $\hat\Delta$ needs to decrease by at least $\theta/4$ from time $t_0$ to $\Tact$.

Note that predictions with values outside $P^{+}$ does not affect $\hat\Delta$. Let $m$ be the number of predictions that fall into $P^{+}$ strictly after epoch $i$ (i.e., during time steps $t_0 + 1, t_0 + 2, \ldots, \Tact$). For each such prediction, suppose that it belongs to the $i'$-th epoch for some $i' > i$. Then, by Algorithm~\ref{algo:scheme}, the bit given by the scheme is drawn from $\Ber(p^*_{i'})$, where $p^*_{i'}$ is the middle point of the interval $\Int_{i'}$ associated with epoch $i'$. Recall that we assumed $\sum_{p\in P\cap\Int_i}\Delta^{+}_p(t_0) \ge \sum_{p\in P\cap\Int_i}\Delta^{-}_p(t_0)$ in this case, so in Algorithm~\ref{algo:scheme}, we let the forecaster place ``$+$'' into the cell in round $i$. For this sign to be preserved in the end, it must hold that the cell chosen in round $i'$ has a larger number than the cell chosen in round $i$. By the choice of $p^*_i$ and $p^*_{i'}$ in the scheme, we have $p^*_{i'} \ge p^*_i + \frac{1}{3k}$. Thus, the contribution of this prediction to $\hat\Delta$ is $b - p$, where $b \sim \Ber(p^*_{i'})$ and $p \in  \Int_i = (p^*_i - \frac{1}{6k}, p^*_i + \frac{1}{6k})$. Thus, the increase in $\hat\Delta$ is bounded between $-1$ and $1$ and has expectation 
    \[p^*_{i'} - p > \left(p^*_i + \frac{1}{3k}\right) - \left(p^*_i + \frac{1}{6k}\right) = \frac{1}{6k}.\]

Therefore, the total increment in $\hat\Delta$ over the $m$ predictions inside $P^{+}$ is the sum of $m$ random variables $X_1, X_2, \ldots, X_m$ satisfying that for each $j \in [m]$: (1) $X_j \in [-1, 1]$ almost surely; (2) $\Ex{}{X_j|X_1, X_2, \ldots, X_{j-1}} \ge \frac{1}{6k}$. Then, applying Lemma~\ref{lem:azuma} gives
\[
    \pr{\hat\Delta(\Tact) - \hat\Delta(t_0) \le -\theta/4}
\le \exp\left(-\frac{1}{2m}\left(\frac{m}{6k} + \frac{\theta}{4}\right)^2\right)
\le \exp\left(-\frac{\theta}{12k}\right).
\]
Finally, by a union bound over the $\le T$ possible values of $m$, the probability that epoch $i$ satisfies all the conditions in Lemma~\ref{lem:covered} is at most $T \cdot \exp\left(-\frac{\theta}{12k}\right)$. Plugging $k = T^{\frac{1}{\alpha + 2\beta + 2}}$ and $\theta = \frac{1}{1440}\sqrt{\frac{T}{k^{\alpha}\ln T}}$ into the bound shows that it is $o(1/T)$.

\paragraph{Case 2: $\sum_{p\in P\cap\Int_i}\Delta^{+}_p(t_0) < \sum_{p\in P\cap\Int_i}\Delta^{-}_p(t_0)$.} This case is analogous to Case 1 and the proof is almost the same. Here we define $P^{-}\coloneqq \{p\in P \cap \Int_i: \Delta_p(t_0) < 0\}$ and $\hat\Delta(t) \coloneqq \sum_{p\in P^{-}}[-\Delta_p(t)]$ instead. Then, for epoch $i$ to be both non-negligible and covered, $\hat\Delta$ needs to be decreased from $\hat\Delta(t_0) \ge \theta / 2$ to $\hat\Delta(\Tact) < \theta/4$. Suppose that exactly $m$ predictions after epoch $i$ fall into set $P^{-}$. Again, assuming that the sign placed in round $i$ is preserved, we can show that the contribution of each such prediction to $\hat\Delta$ is always between $-1$ and $1$ and has expectation at least $\frac{1}{6k}$. Therefore, applying Lemma~\ref{lem:azuma} shows that the probability that $\hat\Delta_t$ decreases by at least $\theta/4$ after $m$ such predictions is exponentially small, which completes the proof.
\end{proof}

\end{document}